\documentclass[runningheads]{llncs}

\usepackage{eccv}

\usepackage{eccvabbrv}

\usepackage{graphicx}
\usepackage{booktabs}

\usepackage[accsupp]{axessibility}  

\usepackage{lipsum}
\usepackage{multirow}
\usepackage{makecell}
\usepackage{amssymb, bm}
\usepackage{siunitx}
\usepackage{tablefootnote}

\usepackage{pifont}

\usepackage[bookmarks=false, linkcolor=blue, urlcolor=blue, citecolor=blue]{hyperref} 
\hypersetup{
    colorlinks=true,
    linkcolor=red,
    filecolor=magenta,      
    urlcolor=blue,
    pdfstartview={FitH},
    citecolor =blue
    }

\usepackage{orcidlink}

\begin{document}

\title{\texorpdfstring{Language-Driven 6-DoF Grasp Detection \\Using Negative Prompt Guidance}{}}

\titlerunning{Language-Driven 6-DoF Grasping w. Negative Prompt Guidance}

\author{Toan Nguyen\inst{1}\orcidlink{0009-0008-0534-647X} \and
Minh Nhat Vu\inst{2,3,*}\orcidlink{0000-0003-0692-8830} \and
Baoru Huang\inst{4}\orcidlink{0000-0002-4421-652X} \and
An Vuong\inst{1}\orcidlink{0009-0003-8533-9897} \and \\
Quan Vuong\inst{5}\orcidlink{0009-0009-1829-9614} \and
Ngan Le\inst{6}\orcidlink{0000-0003-2571-0511} \and
Thieu Vo\inst{7}\orcidlink{0000-0001-7957-5648} \and
Anh Nguyen\inst{8}\orcidlink{0000-0002-1449-211X} 
}

\authorrunning{Nguyen et al.}

\institute{
FPT Software AI Center, Vietnam \and
TU Wien, Austria \and
AIT GmbH, Austria, $^*$Corresponding author \and
Imperial College London, United Kingdom \and
Physical Intelligence, United States \and
University of Arkansas, United States \and
Ton Duc Thang University, Vietnam \and
University of Liverpool, United Kingdom
}
\maketitle

\begin{abstract}
6-DoF grasp detection has been a fundamental and challenging problem in robotic vision. While previous works have focused on ensuring grasp stability, they often do not consider human intention conveyed through natural language, hindering effective collaboration between robots and users in complex 3D environments. In this paper, we present a new approach for language-driven 6-DoF grasp detection in cluttered point clouds. We first introduce Grasp-Anything-6D, a large-scale dataset for the language-driven 6-DoF grasp detection task with 1M point cloud scenes and more than 200M language-associated 3D grasp poses. We further introduce a novel diffusion model that incorporates a new negative prompt guidance learning strategy. The proposed negative prompt strategy directs the detection process toward the desired object while steering away from unwanted ones given the language input. Our method enables an end-to-end framework where humans can command the robot to grasp desired objects in a cluttered scene using natural language. Intensive experimental results show the effectiveness of our method in both benchmarking experiments and real-world scenarios, surpassing other baselines. In addition, we demonstrate the practicality of our approach in real-world robotic applications. Our project is available at \href{https://airvlab.github.io/grasp-anything/}{https://airvlab.github.io/grasp-anything}.
\end{abstract}

\section{Introduction} \label{Sec:Intro}
Grasp detection stands as a foundational and enduring challenge in the field of robotics and computer vision~\cite{caldera2018review,kleeberger2020survey}. This task involves identifying a suitable configuration for the robotic hand that stably grasps the objects, facilitating the effective manipulation capability in the robot's operating environment. Traditional grasp detection methods have predominantly focused on ensuring the stability of the detected grasp pose, while often neglecting the human intention. This limitation underscores a large gap between current approaches and real-world user-specified requirements~\cite{vuong2023grasp}. The integration of human intention conveyed through natural language, is therefore crucial to help robots perform complex tasks more flexibly. This enables users to communicate task specifications more intuitively and comprehensively to the intelligent robot, facilitating a more effective human-robot collaboration.

\vspace{-4ex}
\begin{figure}[h]
    \centering
    \includegraphics[width=1.0\linewidth]{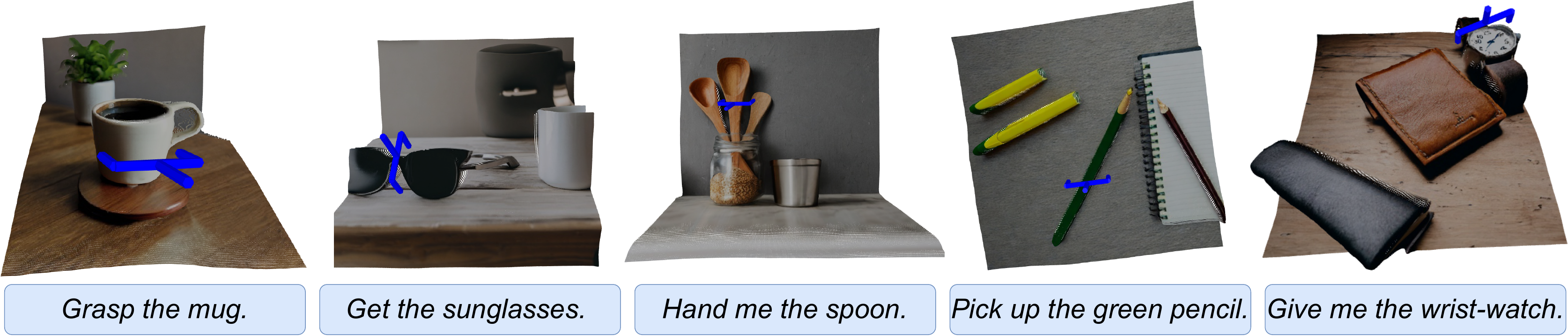}
    \caption{We tackle the task of language-driven 6-DoF grasp detection in cluttered 3D point cloud scenes.}
    \label{Fig:Intro}
\end{figure}
\vspace{-4ex}

In recent years, thanks to advancements in large language models~\cite{devlin2018bert,gpt4,chowdhery2023palm} and large vision-language models~\cite{radford2021learning,li2022blip,li2022grounded}, there has been a surge of interest in language-driven robotics research~\cite{brohan2023can,brohan2022rt,zitkovich2023rt,driess2023palm,rashid2023language,nguyen2023open,van2023open}. This research field focuses on developing intelligent robots that can understand and respond to human linguistic commands. For example, SayCan~\cite{brohan2023can} and PaLM-E~\cite{chowdhery2023palm} are robotic language models designed to provide instructions for robots operating in real-world environments. Trained on large-scale data, RT-1~\cite{brohan2022rt} and RT-2~\cite{zitkovich2023rt} are robotic systems capable of performing low-level actions in response to natural language commands. While significant progress has been made in the field, it is noteworthy that only a few works have addressed the task of language-driven grasp detection~\cite{nguyen2023language,tang2023graspgpt,vuong2023grasp,tang2023ttask,rashid2023language,tziafas2023language}. Furthermore, these methods still exhibit considerable shortcomings. Particularly, while the authors in~\cite{nguyen2023language,tang2023graspgpt} solely focus on single-object scenarios, the works in~\cite{vuong2023grasp,tang2023ttask,tziafas2023language} restrict grasp detection to 2D configurations. These limitations prevent the robot from capturing the complexity of real-world 3D and multi-object scenarios. In this research, we address these limitations by training a new system that detects language-driven 6-DoF grasp poses, with a focus on grasping objects within diverse and complex scenes represented as 3D point clouds.

We first introduce a new dataset, namely \textbf{Grasp-Anything-6D}, as a large-scale dataset for language-driven 6-DoF grasp detection in 3D point clouds. Our dataset builds upon the Grasp-Anything dataset~\cite{vuong2023grasp} and incorporates a state-of-the-art depth estimation method~\cite{bhat2023zoedepth} to support 2D to 3D projection, and manual correction to ensure the dataset quality. Specifically, Grasp-Anything-6D provides one million (1M) 3D point cloud scenes with comprehensive object grasping prompts and dense 6-DoF grasp pose annotations. With its extensive volume, our dataset enables the capability of 6-DoF grasp detection using language instructions directly from the point cloud. Empirical demonstrations show that our dataset successfully facilitates grasp detection in diverse and complex scenes, both in vision-based experiments and real-world robotic settings.

With the new dataset in hand, we propose a new diffusion model to address the challenging problem of language-driven 6-DoF grasp detection called \textbf{LGrasp6D}. We opt for diffusion models due to their recent impressive results in various generation tasks~\cite{ho2020denoising,song2020denoising,nichol2021improved}, including image synthesis~\cite{dhariwal2021diffusion,nichol2021glide}, video generation~\cite{sora2024,wu2023tune}, and point cloud generation~\cite{luo2021diffusion,nakayama2023difffacto}. However, the application of diffusion models to grasp detection remains under-explored~\cite{nguyen2023language,urain2023se}. Unlike previous works that mostly focus on language-driven grasp detection in 2D image~\cite{vuong2023grasp,tang2023ttask,tziafas2023language} or in 3D point cloud with single object~\cite{nguyen2023language,tang2023graspgpt}, our work proposes a new diffusion model for language-driven 6-DoF grasp detection in cluttered 3D point cloud environments. In practice, language-driven 6-DoF grasp detection is a fine-grained task driven by the language, e.g., \textit{``Grasp the blue cup''} and \textit{``Grasp the black cup''} are for two different objects in the scene. Therefore, we introduce a new negative prompt guidance learning strategy to tackle this fine-grained nature. The main motivation of this strategy is to learn a negative prompt embedding that can encapsulate the notion of other undesired objects in the scene. When being applied in the generation process, the learned negative prompt embedding explicitly guides the grasp pose toward the desired object while avoiding unwanted ones. Our LGrasp6D method is an end-to-end pipeline that enables humans to command the robot to grasp desired objects in a cluttered scene using a natural language prompt. Figure~\ref{Fig:Intro} illustrates examples of our language-driven grasp detection in 3D point clouds. To summarize, our contributions are three-fold:
\begin{itemize}
    \item We propose Grasp-Anything-6D, a large-scale dataset for language-driven 6-DoF grasp detection in 3D point clouds.
    \item We propose a new diffusion model that learns and applies negative prompt guidance, significantly enhancing the grasp detection process.
    \item We demonstrate that our dataset and the proposed method outperform other approaches and enable successful real-world robotic manipulation.
\end{itemize}
\section{Related Works} \label{Sec:Related Works}
\textbf{Robot Grasp Detection.}
Several works for robot grasp detection addressed the task on 2D images~\cite{jiang2011efficient,redmon2015real,lenz2015deep,zhang2019roi}. Thanks to recent advancements in 3D perception~\cite{qi2017pointnet,qi2017pointnet++,ni2020pointnet++,kerr2023lerf}, 6-DoF grasp detection in 3D point clouds is gaining increasing interest in both computer vision and robotics communities. In general, two main lines of approaches have been employed for this problem. The first line~\cite{mousavian20196,liang2019pointnetgpd,murali20206,gou2021rgb} involves sampling various grasp candidates across the input point cloud, followed by validation using a grasp evaluator network.
The primary drawback of methods in the first line lies in their inefficiency in terms of speed, attributed to their multi-stage structure. In contrast, the second line of research detects the grasp poses in an end-to-end manner~\cite{fang2020graspnet,ni2020pointnet++,qin2020s4g,Wang_2021_ICCV}, achieving a more favorable balance in terms of the time-accuracy tradeoff. For instance, Qin~\textit{et al.}~\cite{qin2020s4g} presented a novel gripper contact model and a single-shot neural network to predict amodal grasp proposals, while Wang~\textit{et al.}~\cite{Wang_2021_ICCV} proposed the concept of graspness to detect the scene graspable areas. However, most of the existing 6-DoF grasp detection methods do not take into account language as the input. In this work, we follow the end-to-end approach. Our method integrates language instructions into the grasp detection process, ensuring that the detected grasp pose is aligned with the user-specified requirements.

\textbf{Language-Guided Robotic Manipulation.}
Amidst the remarkable strides of large language models~\cite{devlin2018bert,brown2020language,gpt4} and large vision-language models~\cite{radford2021learning,li2022grounded,li2022blip}, several recent works have harnessed language semantics for multiple tasks of robot manipulation~\cite{brohan2022rt,garg2022lisa,nguyen2023open,ren2023leveraging,zitkovich2023rt}. For instance, the authors in~\cite{garg2022lisa} presented a framework that learns meaningful skills from language-based expert demonstrations. Nguyen \textit{et al.}~\cite{nguyen2023open} utilized language to detect open-vocabulary affordance for 3D point cloud objects. More recently, the authors in~\cite{zitkovich2023rt} proposed a family of models that learn generalizable and semantically aware policies derived from fine-tuning large vision-language models trained on web-scale data. Besides, the task of language-guided grasp detection is also under active exploration. However, approaches in this research direction present several limitations. Specifically, the works in~\cite{nguyen2023language,tang2023graspgpt} only addressed single-object scenarios. The authors in~\cite{vuong2023grasp,tang2023ttask,tziafas2023language,vuong2024language} exclusively detected 2D rectangle grasp poses. More recently, the method in~\cite{rashid2023language} required multiple viewpoints of the scene to build the language field, which is not always obtainable. In contrast to these works, our method is capable of detecting language-driven 6-DoF grasp poses in cluttered single-view point cloud scenes, making it well-suitable for real-world robotic applications.

\textbf{Diffusion Probabilistic Models.}
Diffusion models are a class of neural generative models, based on the stochastic diffusion process in Thermodynamics~\cite{sohl2015deep}. In this setting, a sample from the data distribution is gradually noised by the forward diffusion process. Then, a neural network learns the reverse process to gradually denoise the sample. First introduced by~\cite{sohl2015deep}, diffusion models have been further simplified and accelerated~\cite{song2020denoising,ho2020denoising}, and improved significantly~\cite{austin2021structured,wang2024patch,nichol2021improved,song2021maximum}. In recent years, many works have explored applying diffusion models for various generation problems, such as image synthesis~\cite{xie2023smartbrush,dhariwal2021diffusion}, scene synthesis~\cite{vuong2023language,lee2024posediff}, and human motion generation~\cite{wang2023fg,tevet2022human}. In robotics, diffusion models have also been applied to many problems ranging from policy learning~\cite{chi2023diffusion,ze2023gnfactor}, task and motion planning~\cite{liu2022structdiffusion,urain2023se} to robot design~\cite{wang2023diffusebot}. However, few works have adopted diffusion models for the task of grasp detection~\cite{nguyen2023language,urain2023se}. Notably, none of them consider the task of language-driven grasping in 3D cluttered point clouds. To address this challenging task, we propose a novel diffusion model that incorporates a new negative prompt guidance learning approach. This strategy assists in guiding the generation process toward the desired grasp distributions while steering away from unwanted ones. The effectiveness of our proposed approach is demonstrated through comprehensive experiments.
\section{The Grasp-Anything-6D Dataset} \label{Sec:Dataset}
Our Grasp-Anything-6D dataset is built upon the Grasp-Anything dataset~\cite{vuong2023grasp}. Leveraging foundation models~\cite{chatgpt,rombach2022high}, Grasp-Anything is a large-scale dataset for 2D language-driven grasp detection. This dataset consists of 1M RGB images and $\approx$3M objects, substantially surpassing prior datasets in diversity and volume. To bring the problem from 2D to 3D, we first leverage the state-of-the-art depth estimation method ZoeDepth~\cite{bhat2023zoedepth} to estimate the depth map given the input RGB images of Grasp-Anything. Subsequently, we perform projection and manual verification to ensure the quality of our dataset. 

\textbf{3D Scenes and 6-DoF Grasps Construction.}
For a given 2D scene in the Grasp-Anything dataset~\cite{vuong2023grasp}, we first employ ZoeDepth~\cite{bhat2023zoedepth} to get the depth map for the image and establish the 3D point cloud scene with the camera model assumption of a 55-degree field of view and central principal point. We select the field of view of 55 degrees because it leads to 3D scenes representing real object scales. Next, to bring a 2D grasp configuration to 3D, we first infer the 3D position using the center of the 2D rectangle grasp in the image. Since in the Grasp-Anything dataset, the position of the 2D grasp may not necessarily be integers, we employ bilinear interpolation to calculate its corresponding 3D position by considering the 3D coordinates of neighboring pixels. The position determines the translation part of the grasp representation. For the rotation part, we utilize the angle of the 2D rectangle grasp and map it to 3D to rotate the 6-DoF grasp pose accordingly. The width of the 6-DoF grasp is derived from the width of the 2D grasp. Adhering to the Robotiq 2F-140 gripper specifications~\cite{2f140}, we establish the maximum grasp width as \SI{202.1}{\mm}, and discard any grasps exceeding this threshold. The overview of our 3D scenes and 6-DoF grasps construction process is illustrated in Figure~\ref{Fig:dataset_pipeline}. We maintain the same scene description and grasping prompts as in the Grasp-Anything dataset. Additionally, we infer the 3D masks on the point cloud scene for every object in the grasp list using the corresponding segmentation masks in 2D. 

\vspace{-4ex}
\begin{figure}[h]
    \centering
    \includegraphics[width=\linewidth]{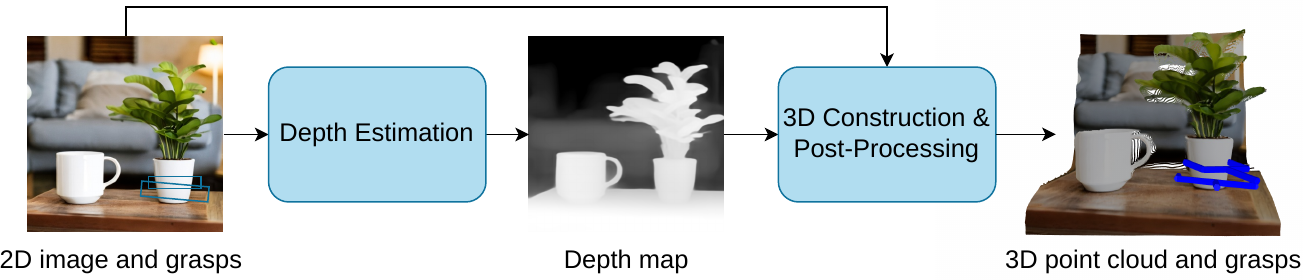}
    \caption{Overview of Grasp-Anything-6D dataset construction pipeline.}
    \label{Fig:dataset_pipeline}
\end{figure}
\vspace{-4ex}

\textbf{Post-Processing.} After converting the 2D scenes and grasps to 3D, we manually check for the collision of the 6-DoF grippers and point cloud scenes, as well as whether the grippers can stably grasp the objects. These problems may occur since the depth estimation network~\cite{bhat2023zoedepth} may not always bring good results. Concretely, we remove the grasp poses that collide with the point cloud scene and those whose closing volume between the fingers does not intersect the object determined by its 3D mask.
As a result, our Grasp-Anything-6D dataset consists of 1M point cloud scenes, with comprehensive grasping prompts, and 200M corresponding dense and high-quality 6-DoF grasp poses.
\vspace{-2ex}
\section{Grasp Detection using Negative Prompt Guidance} \label{Sec:Method}
\vspace{-1ex}
\subsection{Motivation}
Diffusion models have recently shown remarkable performance across various generation tasks. This makes it a promising choice for our problem, where grasp detection can be viewed as a generation process conditioned on both the point cloud scene and the language prompt. The main contribution of our diffusion model is a novel negative prompt guidance learning strategy. This is motivated by the notion that generating a grasp for a specific object can benefit significantly from guidance away from unwanted objects in the scene. Our LGrasp6D leverages this by integrating learning the negative prompt embedding into the training process alongside the conventional denoising objective. Our target for the negative prompt embedding is to capture the notion of other undesired objects in the scene. The learned negative prompt guidance is then applied in the sampling to assist the grasp detection process.
\subsection{Language-Driven 6-DoF Grasp Detection}

\vspace{-4ex}
\begin{figure}[h]
    \centering
    \includegraphics[width=\linewidth]{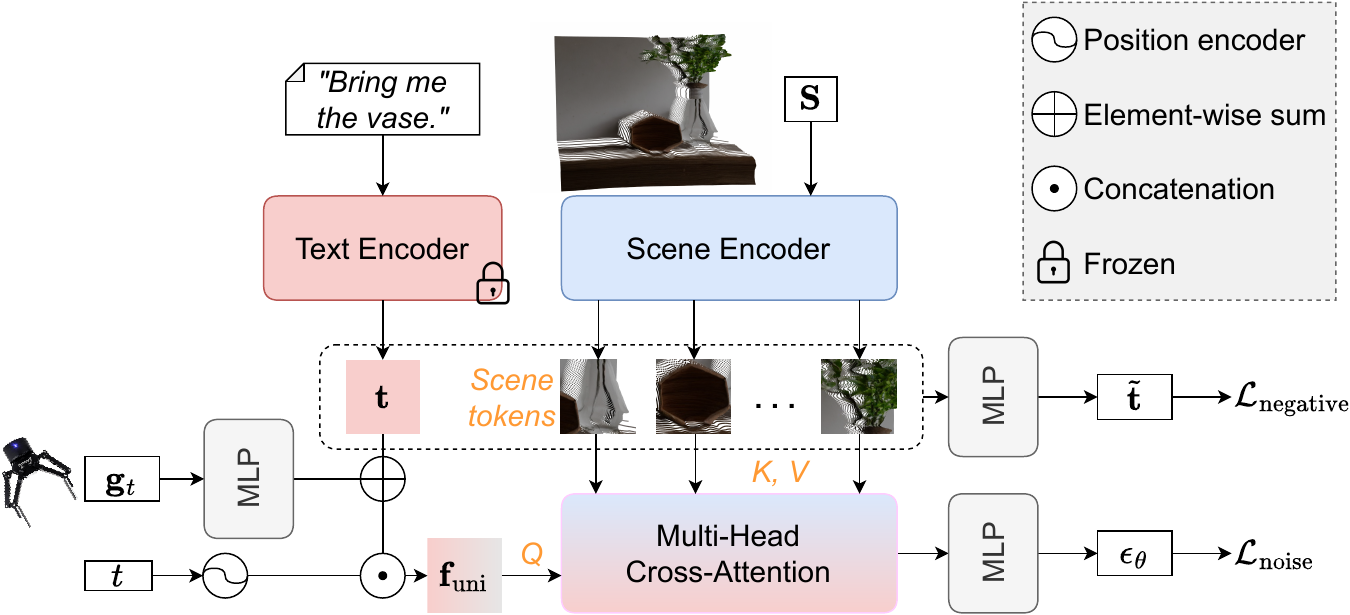}
    \caption{Overview of our denoising network. In addition to predicting the noise, our denoising network is trained to learn the negative prompt embedding, which is supervised by the text embeddings associated with other unwanted objects in the same scene.}
    \label{Fig:Method}
\end{figure}
\vspace{-4ex}

\textbf{Forward Process.}
We use the $\mathfrak{se}(3)$ Lie algebra~\cite{samelson2012notes} to represent the translation and rotation of our grasp poses. We use the $\mathfrak{se}(3)$ representation since it allows us to conveniently perform the operators of addition and multiplication by a scalar required by our forward and reverse diffusion processes. The grasp pose is then represented as the concatenation of $\mathfrak{se}(3)$ vector and the grasp width. Note that one can easily convert between the $\mathfrak{se}(3)$ and $4\times 4$ transformation matrix representation using the logarithm map and exponential map~\cite{onishchik2012lie}.
Given a target grasp pose $\mathbf{g}_0$ in the training dataset, in the forward process, we obtain a sequence of perturbed grasp poses by gradually adding to it small amounts of Gaussian noise in $T$ steps. The noise step sizes are specified by a predefined variance schedule $\left \{ \beta_t \in (0, 1)\right \}^T_{t=1}$.
The forward process is formulated as:

\begin{equation}
    q\left( \mathbf{g}_t | \mathbf{g}_{t-1} \right ) = \mathcal{N}\left ( \mathbf{g}_t;\sqrt{1-\beta_t}\mathbf{g}_{t-1}, \beta_t\mathbf{I} \right ).
\end{equation}
The perturbed pose at any arbitrary time step $t$ can be obtained by:

\begin{equation}    
\mathbf{g}_t = \sqrt{\bar{\alpha}_t}\mathbf{g}_0 + \sqrt{1-\bar{\alpha}_t}\bm{\epsilon},
\end{equation}
where $\bar{\alpha}_t =\prod_{i=1}^{t}\alpha_t$ with $\alpha_t = 1-\beta_t$ and $\bm{\epsilon} \sim \mathcal{N}\left (\mathbf{0}, \mathbf{I} \right )$. When $T\rightarrow \infty$, $\mathbf{g}_T$ is equivalent to $\mathcal{N}\left (\mathbf{0}, \mathbf{I} \right )$~\cite{ho2020denoising}.

\textbf{Denoising Network.}
Our denoising network approximates the added noise described in the forward process by incorporating both the conditions of the point cloud scene and the textual prompt specifying the target object. Additionally, our network learns a vector representation serving as a negative prompt guidance. In our framework, this representation is guided by the available textual prompts associated with other objects within the scene. The details of our denoising network are shown in Figure~\ref{Fig:Method}.

The denoising network first encodes the grasp pose $\mathbf{g}_t$ at a specific time step $t$ using a grasp encoder MLP. The scene encoder encodes the point cloud scene $\mathbf{S}$ to $n_s$ scene embedding tokens. In our framework, we use PointNet++~\cite{qi2017pointnet++} as the underlying architecture for the scene encoder. For the textual prompt, we employ a pretrained text encoder to get a text embedding $\mathbf{t}$. We use sinusoidal positional embedding~\cite{ho2020denoising} to embed the time step $t$ to a high-dimensional vector. Afterward, we form the unified representation $\mathbf{f}_{\text{uni}}$ of the grasp pose, the textual prompt, and the time step. In concrete, we concatenate the time embedding with the element-wise sum of the grasp embedding and the text embedding. Subsequently, we adopt the multi-head cross-attention mechanism to capture the intricate relationships among input components. Specifically, the query for the cross-attention is the unified feature $\mathbf{f}_{\text{uni}}$ while the $n_s$ scene tokens serve as keys and values. The output of the cross-attention module is then fed to an MLP to obtain the predicted noise $\bm{\epsilon}_{\bm{\theta}}\left ( \mathbf{g}_t,\mathbf{S},\mathbf{t},t\right )$.
We supervise the noise prediction by optimizing the simplified objective function as described in~\cite{ho2020denoising}:
\begin{equation}
    \mathcal{L}_{\text{noise}} = \mathbb{E}_{\bm{\epsilon},\mathbf{g}_0, \mathbf{S}, \mathbf{t}, t}\left [ \left \| \bm{\epsilon}_{\bm{\theta}}\left ( \mathbf{g}_t,\mathbf{S},\mathbf{t},t\right ) - \bm{\epsilon} \right \|^2\right].
    \label{Eq:NoiseLoss}
\end{equation}

\textbf{Negative Prompt Learning.} Along with estimating the noise, the denoising network also produces the negative prompt embedding $\tilde{\mathbf{t}}$. We subtract the text embedding $\mathbf{t}$ from the scene tokens, compute the average over $n_s$ resulting vectors, and then pass the output through an MLP to get $\tilde{\mathbf{t}}$. Our purpose for $\tilde{\mathbf{t}}$ is that it can encapsulate the notion of other objects in the same scene. Hence, our objective is to minimize the distance between $\tilde{\mathbf{t}}$ and the negative text embeddings which are text embeddings corresponding to other objects. Specifically, we define the loss function for the learning of negative prompt embedding as:
\begin{equation}
    \mathcal{L}_{\text{negative}} = D\left(\tilde{\mathbf{t}}, \bar{\mathbf{T}}= \left \{ \bar{\mathbf{t}}_i \right \}_{i=1}^{m} \right ) = \text{min}_{i=1}^m\left \| \tilde{\mathbf{t}}-\bar{\mathbf{t}}_i \right \|_2^2,
    \label{Eq:NegLoss}
\end{equation}
where $D\left (\cdot \right )$ denotes the distance function, $\bar{\mathbf{T}}= \left \{ \bar{\mathbf{t}}_i \right \}_{i=1}^{m}$ is the set of $m$ negative text embeddings. In training, we simultaneously optimize both the denoising loss $\mathcal{L}_{\text{noise}}$ and the loss for negative prompt embedding learning $\mathcal{L}_{\text{negative}}$.

\textbf{Reverse Process with Negative Prompt Guidance.}
Different from conventional diffusion models, our reverse diffusion process utilizes the negative prompt embedding learned during the training to guide the grasp pose toward the desired object while avoiding unwanted ones. 
Our generation process can be formulated as a conditional distribution $p\left ( \mathbf{g} | \mathbf{S},\mathbf{t},\neg\tilde{\mathbf{t}}\right )$. The negation sign of $\tilde{\mathbf{t}}$ indicates that we aim to sample the grasp pose with the absence of the $\tilde{\mathbf{t}}$ prompt condition. 
We begin with the following proposition:
\begin{proposition}
\label{Prop:prop_1}
    The conditional distribution $p\left ( \mathbf{g} | \mathbf{S},\mathbf{t},\neg\tilde{\mathbf{t}}\right )$ can be factorized as
    \begin{equation}
    \label{Eq:propto}
    p\left ( \mathbf{g} | \mathbf{S},\mathbf{t},\neg\tilde{\mathbf{t}}\right ) \propto p\left (\mathbf{g}|\mathbf{S} \right ) \frac{p\left ( \mathbf{g}|\mathbf{t},\mathbf{S}\right )}{p\left ( \mathbf{g}|\tilde{\mathbf{t}},\mathbf{S}\right )}.
    \end{equation}
\end{proposition}
\begin{proof}
    See Supplementary Material.
\end{proof}

With Equation~\ref{Eq:propto}, alongside detecting grasps conditioning on the scene and the user-specified prompt via $p\left (\mathbf{g}|\mathbf{S} \right )$ and $p\left ( \mathbf{g}|\mathbf{t},\mathbf{S}\right )$, we can now seamlessly incorporate the negative prompt guidance into our reverse process via $p\left ( \mathbf{g}|\tilde{\mathbf{t}},\mathbf{S}\right )$.
\begin{remark}
Liu~\textit{et al.}~\cite{liu2022compositional} demonstrated how diffusion models can be composed based on their connection to energy-based models~\cite{du2020compositional}.
We recall this relationship in detail in our Supplementary.
Consequently, following the expression in~\cite{liu2022compositional}, we can formulate our compositional denoising step in the reverse process as:
\begin{equation}
    \tilde{\bm{\epsilon}}_{\bm{\theta}}\left ( \mathbf{g}_t, \mathbf{S}, \mathbf{t},\neg\tilde{\mathbf{t}}, t \right ) = \bm{\epsilon}_{\bm{\theta}}\left (\mathbf{g}_t,\mathbf{S},\varnothing,t \right ) + w\left ( \bm{\epsilon}_{\bm{\theta}}\left (\mathbf{g}_t,\mathbf{S},\mathbf{t},t \right ) - \bm{\epsilon}_{\bm{\theta}}\left (\mathbf{g}_t,\mathbf{S},\tilde{\mathbf{t}},t \right )\right ).
    \label{Eq:denoisingstep}
\end{equation}
\end{remark}
In Equation~\ref{Eq:denoisingstep}, $p\left (\mathbf{g}|\mathbf{S} \right )$, $p\left ( \mathbf{g}|\mathbf{t},\mathbf{S}\right )$ and $p\left ( \mathbf{g}|\tilde{\mathbf{t}},\mathbf{S}\right )$ are parameterized by $\bm{\epsilon}_{\bm{\theta}}\left (\mathbf{g}_t,\mathbf{S},\varnothing,t \right )$, $\bm{\epsilon}_{\bm{\theta}}\left (\mathbf{g}_t,\mathbf{S},\mathbf{t},t \right )$ and $\bm{\epsilon}_{\bm{\theta}}\left (\mathbf{g}_t,\mathbf{S},\tilde{\mathbf{t}},t \right )$ respectively. $\bm{\epsilon}_{\bm{\theta}}\left (\mathbf{g}_t,\mathbf{S},\tilde{\mathbf{t}},t \right )$ is the output of the denoising network when the learned negative prompt embedding $\tilde{\mathbf{t}}$ is plugged in as the text embedding. $w$ is a hyperparameter that controls the strength of the negative guidance.
$\bm{\epsilon}_{\bm{\theta}}\left (\mathbf{g}_t,\mathbf{S},\varnothing,t \right )$ is the predicted noise when the text condition is discarded. In training, we learn $\bm{\epsilon}_{\bm{\theta}}\left (\mathbf{g}_t,\mathbf{S},\varnothing,t \right )$ by randomly masking out the text embedding with a predefined probability $p_{\text{mask}}$. Given the denoising step defined in Equation~\ref{Eq:denoisingstep}, we can now sample grasps from Gaussian noise by applying the reverse process from timestep $T$ back to $0$ using the following formulation:
\begin{equation}
    \mathbf{g}_{t-1} = \frac{1}{\sqrt{\alpha_t}}\left ( \mathbf{g}_t - \frac{1-\alpha_t}{\sqrt{1-\bar{\alpha}_t}}\tilde{\bm{\epsilon}}_{\bm{\theta}}\left ( \mathbf{g}_t, \mathbf{S}, \mathbf{t},\neg\tilde{\mathbf{t}}, t \right ) \right ) + \sqrt{\beta_t}\mathbf{z},
\end{equation}
where $\mathbf{z} \sim \mathcal{N}\left (\mathbf{0}, \mathbf{I} \right )$ if the time step $t>1$, else $\mathbf{z} = \mathbf{0}$.

\subsection{Training and Sampling}
We define the overall loss function for training as $\mathcal{L} = 0.9 \mathcal{L}_{\text{noise}} + 0.1 \mathcal{L}_{\text{negative}}$. We utilize the pretrained CLIP ViT-B/32 text encoder~\cite{radford2021learning} for our text encoder and freeze it during training. We set the number of timesteps to $T=200$, and set the forward diffusion variances to increase linearly from $\beta_1 = 10^{-4}$ to $\beta_T=0.02$. The probability of masking out the text embedding is set to $p_{\text{mask}} = 0.1$. The whole network is trained over 200 epochs on a cluster of 8 A100 GPUs with a batch size of 128. We use Adam optimizer~\cite{kingma2014adam} with the learning rate $10^{-3}$ and the weight decay $10^{-4}$.
In sampling, we set the negative guidance scale to $w=0.2$. To obtain a favorable inference speed, we pre-compute the scene tokens, the text embedding $\mathbf{t}$, and the negative prompt embedding $\tilde{\mathbf{t}}$ since they are independent of the timestep. This precomputation substantially reduces the detection time, making our method feasible for practical implementation on real robots.
\section{Experiments} \label{Sec:Experiments}
In this section, we evaluate the effectiveness of our LGrasp6D trained on the Grasp-Anything-6D dataset via several vision-based and real robot experiments.

\subsection{Language-Driven 6-DoF Grasp Detection Results}

\textbf{Baselines.}
We evaluate our method against generative approaches for 6-DoF grasp detection, which are 6-DoF GraspNet~\cite{mousavian20196}, SE(3)-DF~\cite{urain2023se}, and 3DAPNet~\cite{nguyen2023language}.
We adapt the frameworks of these baselines to integrate textual input into the detection process. To ensure a fair comparison, we utilize the CLIP ViT-B/32~\cite{radford2021learning} as the text encoder for all methods. We also include our method without utilizing negative prompt guidance (denotes as Ours w.o. NPG) as an additional baseline for comparison.
Detailed implementation information for all baselines is available in our Supplementary Material.

\textbf{Setup.}
We train all baselines on 80\% scenes of the Grasp-Anything-6D dataset and evaluate them on the remaining 20\%. For each pair of point cloud scene-textual prompts, we detect 64 grasp poses for evaluation. To benchmark the methods' detection capabilities, we use three metrics, which are the coverage rate~\cite{mousavian20196}, earth mover's distance~\cite{urain2023se}, and collision-free rate~\cite{zhao2021regnet}. The coverage rate (CR)~\cite{mousavian20196} measures how well the space of ground-truth grasps is covered by the detected grasps. The earth mover's distance (EMD)~\cite{urain2023se} evaluates the dissimilarity between the distributions of ground-truth grasps and the detected ones. Finally, the collision-free rate (CFR)~\cite{zhao2021regnet} assesses the occurrence of collisions between the gripper of the detected grasps and the scene. The final results for all metrics are averaged across all scene-text prompt pairs. 
Since latency is a critical factor for any robotics applications, we additionally benchmark the inference speeds of all methods using the inference time in seconds (IT). Specifically, for each baseline, we calculate its inference time for detecting 1000 grasp poses across 1000 different scene-text pairs and take the average result.

\textbf{Quantitative Results.}
Table~\ref{Tab:6dga} shows the results of language-driven 6-DoF grasp detection on our Grasp-Anything-6D dataset. The outcomes indicate the advantages of our methods, even without negative prompt guidance, over other baselines. Our complete method consistently achieves the highest scores across all three metrics for grasp detection capability. It significantly surpasses the second-best method, which is our framework without negative prompt guidance, with large margins of $0.1235$ on CR, $0.2249$ on EMD, and $0.0370$ on CFR. This highlights the effectiveness of our proposed negative prompt guidance learning. Regarding latency, our methods achieve competitive IT scores compared to other diffusion model-based methods (SE(3)-DF and 3DAPNet). Although 6-DoF GraspNet achieves the best IT, it is important to note that this is a variational autoencoder-based method requiring only a single decoding step, and its results on the remaining metrics are poor.

\vspace{-3ex}
\begin{table}[ht]
\renewcommand\arraystretch{1.2}
\renewcommand{\tabcolsep}{5.5pt}
\small
\centering
\begin{tabular}{l|cccc}\Xhline{1.0pt}
\textbf{Baseline} & \textbf{CR}$\uparrow$ & \textbf{EMD}$\downarrow$ & \textbf{CFR}$\uparrow$ & \textbf{IT}$\downarrow$ \\ \hline
6-DoF GraspNet~\cite{mousavian20196} & 0.3802 & 0.8035 & 0.6900 & \bf{0.4216} \\ 
SE(3)-DF~\cite{urain2023se} & 0.4290 & 0.7565 & 0.7325 & 1.7233 \\
3DAPNet~\cite{nguyen2023language} & 0.4777 & 0.7381 & 0.7213 &  3.4274 \\ \hline

LGrasp6D (ours) w.o. NPG & \underline{0.5459} & \underline{0.6262} & \underline{0.7336} & \underline{1.4328} \\
LGrasp6D (ours) & \bf{0.6694} & \bf{0.4013} & \bf{0.7706} & 1.4832 \\  
\Xhline{1.0pt}
\end{tabular}
\vspace{1ex}
\caption{Results on Grasp-Anything-6D dataset.}
\label{Tab:6dga}
\end{table}
\vspace{-5ex}

\textbf{Qualitative Results.}
We present the qualitative results of all baselines in detecting language-driven grasps in Figure~\ref{Fig:Qual}. Point cloud scenes are selected from our Grasp-Anything-6D dataset. The results indicate that LGrasp6D exhibits a significantly stronger capability in detecting language-driven grasp poses compared to the others. Specifically, our method excels at focusing on the desired objects, whereas other methods often get distracted by undesired ones. More qualitative results are provided in our Supplementary Material.

\begin{figure}[ht]
    \centering
    \includegraphics[width=0.96\linewidth]{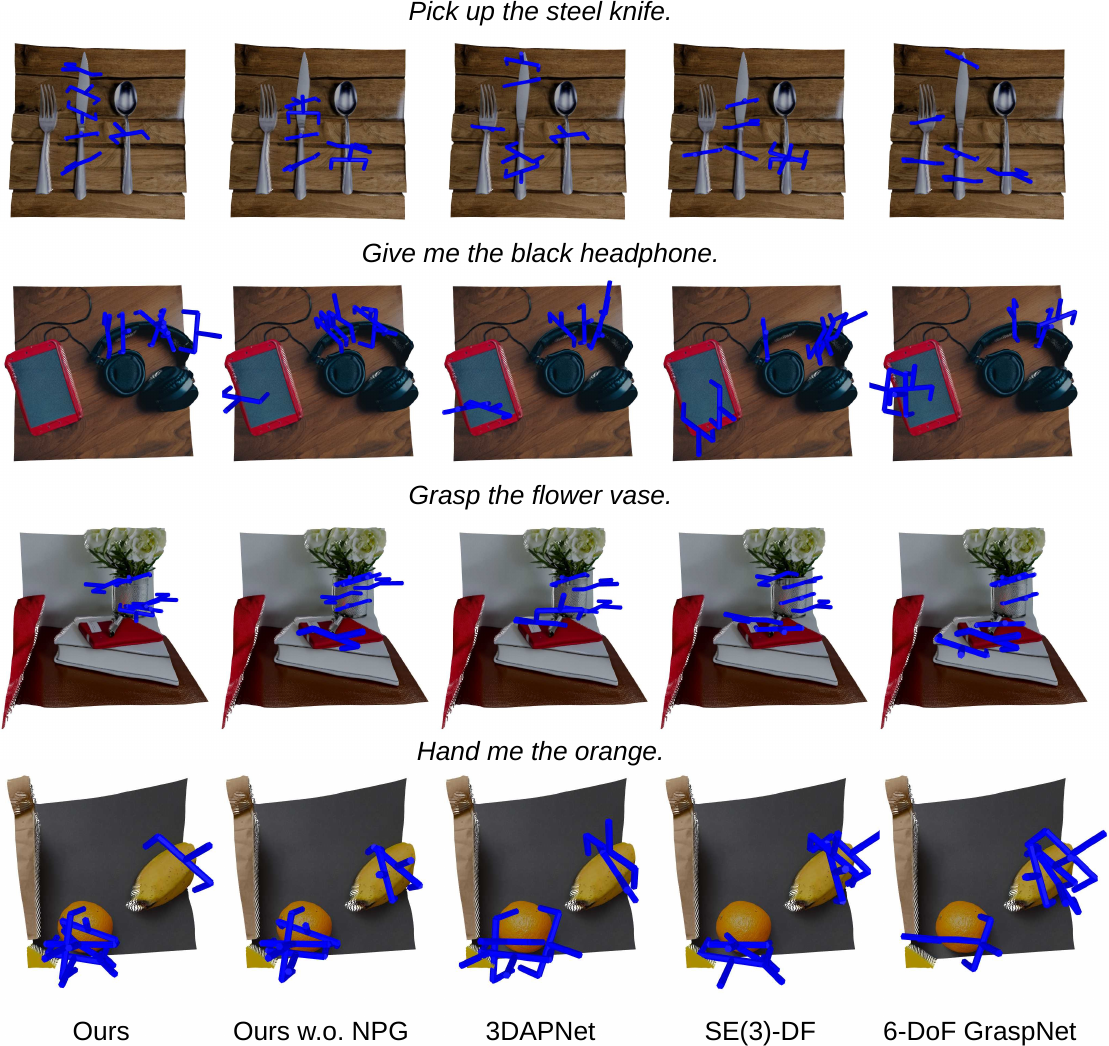}
    \vspace{-1ex}
    \caption{Language-driven 6-DoF grasp detection qualitative results.}
    \label{Fig:Qual}
\vspace{-5ex}
\end{figure}

\begin{table}[b]
\renewcommand\arraystretch{1.2}
\renewcommand{\tabcolsep}{4.0pt}
\small
\centering
\parbox{0.5\linewidth}{
\resizebox{0.5\textwidth}{!}{
\begin{tabular}{l|cccc}\Xhline{1.0pt}
\textbf{Baseline} & \textbf{CR}$\uparrow$ & \textbf{EMD}$\downarrow$ & \textbf{CFR}$\uparrow$ & \textbf{IT}$\downarrow$ \\ \hline
LGrasp6D - 10 steps & 0.5611 & 0.5273 & 0.7368 & \bf{0.0726} \\
LGrasp6D - 20 steps & 0.6425 & 0.4300 & 0.7580 & \underline{0.1464} \\
LGrasp6D - 50 steps & 0.6439 & 0.4254 & 0.7639 & 0.3991 \\
LGrasp6D - 100 steps & \underline{0.6522} & \underline{0.4110} & \underline{0.7633} & 0.8427 \\
LGrasp6D - 200 steps & \bf{0.6694} & \bf{0.4013} & \bf{0.7706} & 1.4832 \\ 
\Xhline{1.0pt}
\end{tabular}
}
\vspace{1ex}
\caption{DDIM accelerating results.}
\label{Tab:accelerating}
}
\hfill
\parbox{0.46\linewidth}{
\resizebox{0.46\textwidth}{!}{
\begin{tabular}{l|ccc}\Xhline{1.0pt}
\textbf{Baseline} & \textbf{CR}$\uparrow$ & \textbf{EMD}$\downarrow$ & \textbf{CFR}$\uparrow$ \\ \hline
6-DoF GraspNet~\cite{mousavian20196} & 0.3498 & 0.8501 & 0.6927 \\ 
SE(3)-DF~\cite{urain2023se} & 0.3892 & 0.7622 & 0.7205 \\
3DAPNet~\cite{nguyen2023language} & 0.4491 & 0.7434 & 0.7092 \\ \hline

LGrasp6D (ours) w.o. NPG & \underline{0.5208} & \underline{0.6422} & \underline{0.7385} \\
LGrasp6D (ours) & \bf{0.6420} & \bf{0.4197} & \bf{0.7683} \\ 
\Xhline{1.0pt}
\end{tabular}
}
\vspace{1ex}
\caption{Cross-dataset results.}
\label{Tab:ctgn}
}
\end{table}

\textbf{Accelerating Detection.} While latency is critical for robot applications, diffusion models are notorious for their low inference speed~\cite{song2023consistency}. Despite our method achieving a competitive inference speed, as shown in Table~\ref{Tab:6dga}, we continue to seek even faster models with comparable performance. Hence, we benchmark our LGrasp6D employing the fast reversion technique of denoising diffusion implicit models (DDIM)~\cite{song2020denoising}, with numbers of sampling steps of 200 (the original one), 100, 50, 20, and 10. The results are shown in Table~\ref{Tab:accelerating}. We can observe that decreases in the sampling step lead to decreases in performance. However, all the variants still outperform other baselines in Table~\ref{Tab:6dga}. Regarding the inference time, these accelerated models obtain significantly better inference speed compared to the original one. The variant with 50 steps already surpasses the 6-DoF GraspNet method (0.3991 seconds compared to 0.4216 seconds). Although the variant with 10 steps achieves the best detection speed, it is not recommended as its detection performance is severely compromised.



\subsection{Generalization Analysis}
\textbf{Cross-Dataset Transferability.}
Given the extensive scale and diversity of our Grasp-Anything-6D dataset, we expect that our proposed method, trained on this dataset, will exhibit strong generalization capabilities when tested on a distinct dataset. Specifically, we evaluate the language-driven grasp detection performance of models trained on Grasp-Anything-6D using the Contact-GraspNet dataset~\cite{sundermeyer2021contact}. This dataset comprises point cloud scenes of cluttered tabletops synthesized using objects and 6-DoF grasps from~\cite{eppner2021acronym} and a random camera view. We utilize the object category names as textual prompts for language-driven grasping. The findings showcased in Table~\ref{Tab:ctgn} exhibit a comparable trend to those observed in the Grasp-Anything-6D dataset. Our method continues to outperform its counterparts across all three metrics, with the version lacking negative prompt guidance following behind. Furthermore, the slight performance decrease on the new dataset is noteworthy. They underscore the efficacy of our dataset, as models trained on it demonstrate strong generalizability.


\textbf{Grasp Detection in the Wild.}
Figure~\ref{Fig:in the wild} illustrates results of our method in point cloud scenes captured from diverse real-world environments, such as working desks, bathrooms, and kitchens. As we can observe, the detected grasp poses exhibit satisfactory quality. This indicates that despite being trained on synthetic data, our approach effectively generalizes to real-world environments.

\vspace{-4ex}
\begin{figure}[ht]
    \centering
    \includegraphics[width=\linewidth]{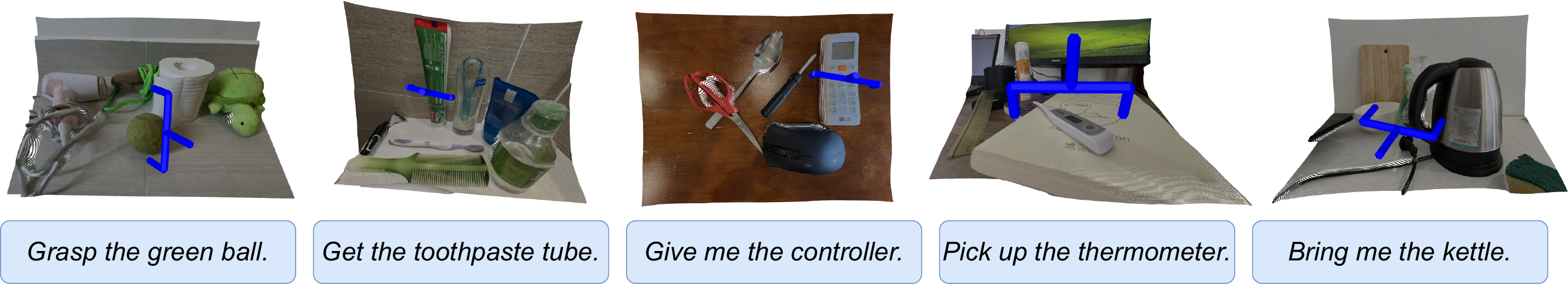}
    \vspace{-3ex}
    \caption{In the wild language-driven 6-DoF grasp detection results.}
    \label{Fig:in the wild}
\end{figure}
\vspace{-6ex}

\subsection{Negative Prompt Guidance Analysis}
We offer a more intuitive understanding of how negative prompt guidance influences the grasp detection results. Specifically, we ultimately sample 1000 grasp poses for each object in a given point cloud scene for both cases: our framework with negative prompt guidance and the one without it. We then employ t-SNE~\cite{van2008visualizing} to visualize all grasp poses on a 2D plane. The results are depicted in Figure~\ref{Fig:npg}, where grasp data points of the same color are detected for the same object. We can observe that negative prompt guidance significantly assists our method in discriminatively detecting grasp poses for different objects. Conversely, without negative prompt guidance, detecting grasp poses for one object is seriously confused by other ones. This further highlights the effectiveness of our proposed approach. More comparison results can be viewed in Figure~\ref{Fig:npg compare}.

\vspace{-3ex}
\begin{figure}[ht]
    \centering
        \begin{minipage}[b]{0.5\textwidth}
        \includegraphics[width=\linewidth]{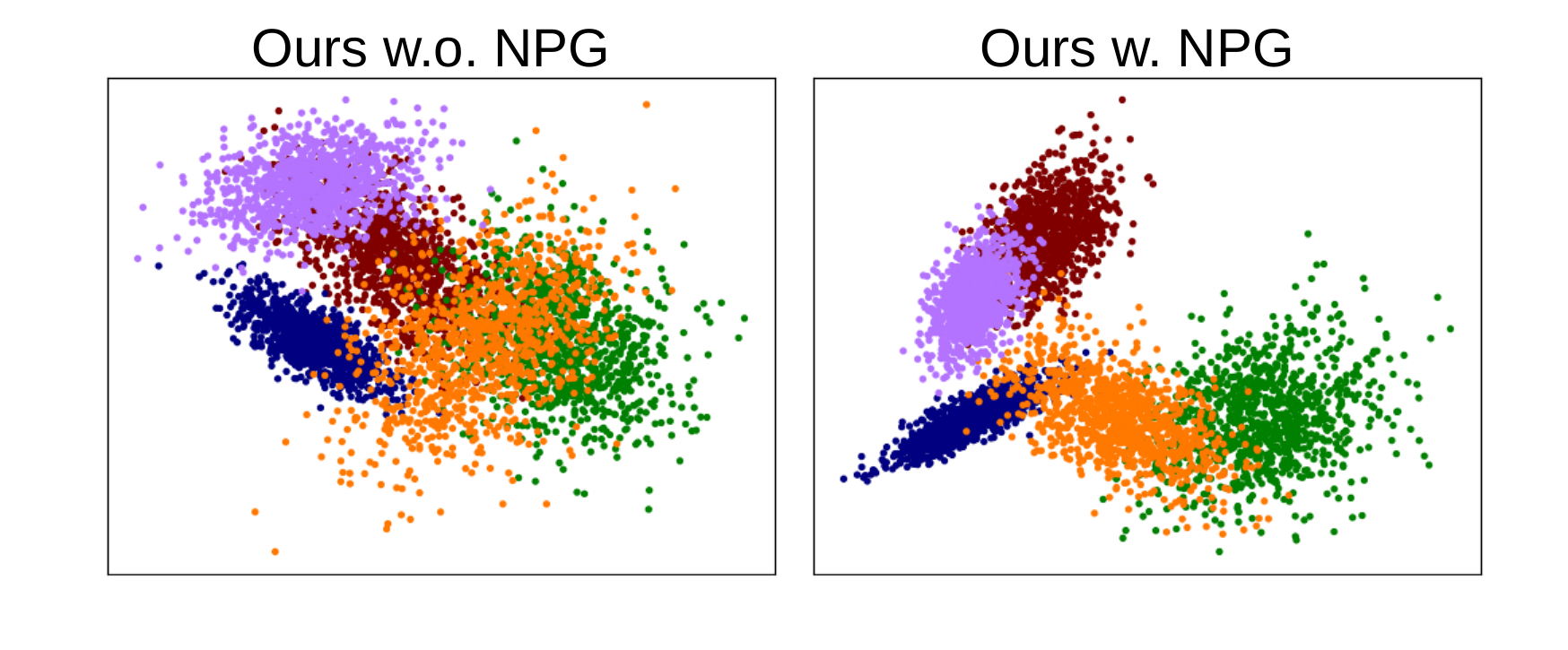}
        \vspace{-3ex}
        \caption{Negative prompt guidance analysis.}
        \label{Fig:npg}
    \end{minipage}%
    \hfill
    \begin{minipage}[b]{0.48\textwidth}
        \includegraphics[width=\linewidth]{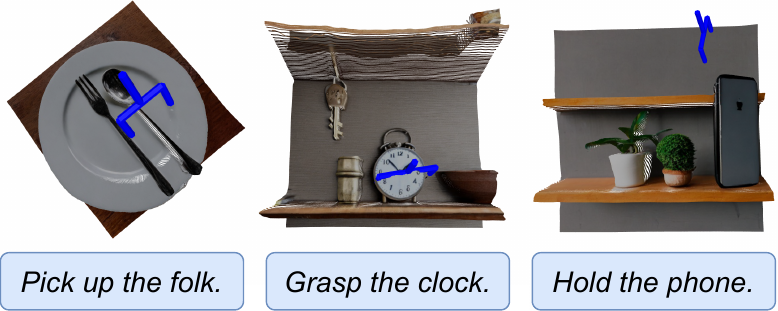}
        \vspace{-3ex}
        \caption{Failure cases.}
        \label{Fig:fail}
    \end{minipage}
\end{figure}
\vspace{-4ex}

\vspace{-4ex}
\begin{figure}[ht]
    \centering
    \includegraphics[width=0.9\linewidth]{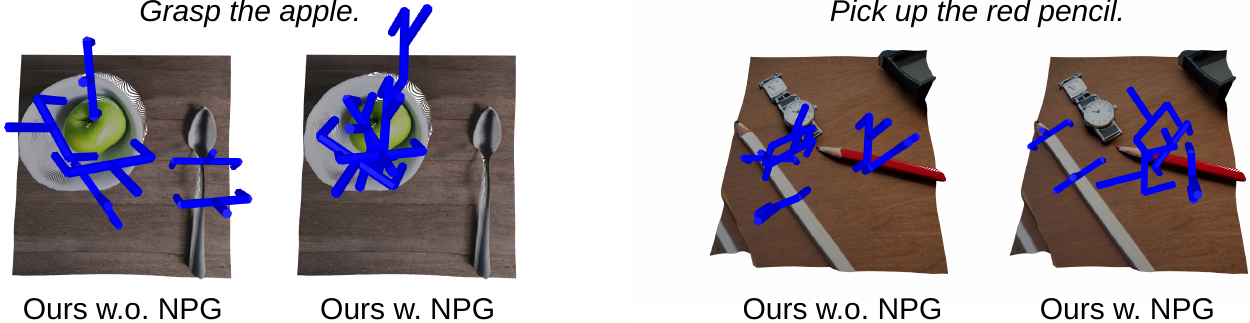}
    \vspace{-1ex}
    \caption{Comparisons between models with and without negative prompt guidance.}
    \label{Fig:npg compare}
\end{figure}


    



\subsection{Robotics Experiment}

\vspace{-4.5ex}
\begin{figure}[ht]
\centering
\def\svgwidth{1\columnwidth}
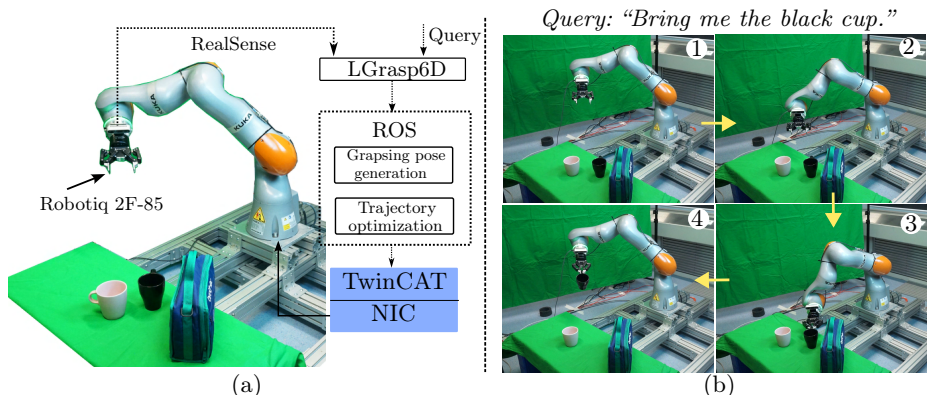
\vspace{-5ex}
\caption{(a) Experiment setup. (b) Example of the execution of a grasping task.}
\label{fig: robot demonstration}
\end{figure}
\vspace{-4ex}

\vspace{-5ex}
\begin{table}[ht]
\renewcommand\arraystretch{1.2}
\renewcommand{\tabcolsep}{4.5pt}
\small
\centering
\begin{tabular}{r|c|cc}\Xhline{1.0pt}
\bf{Baseline} & \textbf{Input Modality} & \textbf{Single} & \textbf{Cluttered} \\ \hline
GG-CNN~\cite{morrison2018closing} + CLIP~\cite{radford2021learning} & RGB-D &0.10  & 0.07 \\
CLIPORT~\cite{shridhar2022cliport} & RGB-D & 0.27 & 0.30 \\
Det-Seg-Refine~\cite{ainetter2021end} + CLIP~\cite{radford2021learning} & RGB-D &0.30  & 0.23\\ 
GR-ConvNet~\cite{kumra2020antipodal} + CLIP~\cite{radford2021learning} & RGB-D  &0.33  & 0.30\\ 
CLIP-Fusion~\cite{xu2023joint} & RGB-D & {0.40} & {0.40} \\
LGD~\cite{vuong2024language} & RGB-D & \textbf{0.43} & \textbf{0.42} \\ \hline
6-DoF GraspNet~\cite{mousavian20196} & Point clouds & 0.31 & 0.27  \\ 
SE(3)-DF ~\cite{urain2023se}  & Point clouds & 0.35 & 0.34 \\
3DAPNet~\cite{nguyen2023language}  & Point clouds & 0.36 & 0.34  \\ \hline
LGrasp6D (ours) w.o. NPG  & Point clouds & 0.38 & 0.36  \\ 
LGrasp6D (ours) & Point clouds & \textbf{0.43} & \textbf{0.42}  \\
\Xhline{1.0pt}
\end{tabular}
\vspace{1ex}
\caption{Robotic language-driven grasp detection results.}
\label{table: real-robot-language-driven}
\end{table}
\vspace{-6ex}

\textbf{Setup.} In Figure~\ref{fig: robot demonstration}, we present the robotic experiment conducted on a KUKA robot. The success rate is used for evaluation.
Using an Intel RealSense D435i depth camera, the detected 6-DoF grasp poses are mapped to robot's 6-DoF end-effector poses using transformation matrices obtained via hand-eye calibration \cite{nguyen2018covariance}. 
The trajectory planner and the computed torque controller~\cite{beck2023singularity,vu2023machine} are employed for the grasp execution.
We use two computers for the experiment. The first computer executes the real-time control software Beckhoff TwinCAT of \SI{8}{\kilo\hertz} update frequency, while the second one utilizes the Robot Operating System (ROS) for the camera and the Robotiq 2F-85 gripper. Using EtherCAT protocol, PC1 communicates with the robot via a network interface card (NIC). The inference process is performed on PC2, utilizing an NVIDIA RTX 3080 graphic card. Our assessment encompasses both single-object and cluttered scenarios, involving a diverse set of real-world daily objects. To ensure the reliability, we repeat each experiment for all methods a total of 45 times.

\textbf{Baselines.} Besides the baselines utilized in previous experiments, we additionally compare LGrasp6D with language-supported versions of state-of-the-art 2D grasp detectors, including GR-CNN~\cite{kumra2020antipodal}, Det-Seg-Refine~\cite{ainetter2021end}, GG-CNN~\cite{morrison2018closing}, CLIPORT~\cite{shridhar2022cliport}, CLIP-Fusion~\cite{xu2023joint}, and LGD~\cite{vuong2024language}. In all cases, we employ the pretrained CLIP ViT-B/32~\cite{radford2021learning} as the text encoder. The implementation details of all baselines can be found in our Supplementary Material.


\textbf{Results.} Our method, incorporating negative prompt guidance, demonstrates better performance compared to other baselines in Table~\ref{table: real-robot-language-driven}. Additionally, even though LGrasp6D is trained on Grasp-Anything-6D, a synthesis dataset exclusively created by foundation models, it still yields commendable results when applied to real-world objects.



\section{Discussion} \label{Sec:Discussion}
Despite promising results, it is important to acknowledge that our method still has limitations, as illustrated in Figure~\ref{Fig:fail}. The left case depicts an example of grasping the wrong object, while the middle one illustrates a detected grasp colliding with an object. The final case shows our method detecting a grasp that mis-targets the desired object. These underscore the challenges in language-driven 6-DoF grasping, indicating its need for further investigation.

For future research, we aim to enhance the performance by incorporating more advanced techniques to capture the intricate correlation among input modalities. In addition, our work can be extended to address language-driven 6-DoF grasping at both the part-level and task-level. For instance, instead of object-specific prompts like \textit{``Grasp the knife''}, one can provide more detailed prompts such as \textit{``Grasp the knife by its handle''} or \textit{"Grasp the knife for cutting''}.  Furthermore, it would be beneficial to extend our approach to accommodate different types of robot end-effectors to enhance the flexibility and adaptability of our framework. Lastly, integrating learning language-driven 6-DoF grasp detection with robotic control could create a more effective end-to-end pipeline, connecting human instructions directly to low-level robot actions.
\section{Conclusion} \label{Sec:Conclusion}
We address the task of language-driven 6-DoF grasp detection in cluttered point clouds. In particular, we have presented the Grasp-Anything-6D dataset as a large-scale dataset for the task with 1M point cloud scenes. We have introduced a novel LGrasp6D diffusion model incorporating the new concept of negative prompt guidance learning. Our proposed negative prompt guidance assists in tackling the fine-grained challenge of the language-driven grasp detection task, directing the detection process toward the desired object by steering away from undesired ones. Empirical results demonstrate the superiority of our method over other baselines in various settings. Furthermore, extensive experiments validate the efficacy of our approach in real-world environments and robotic applications. 

\clearpage

\newpage
\appendix
\section{Theoretical Findings}
\label{Sec: Theory}

\subsection{Proof of Proposition 1}
\begin{proof}

We have the following derivation:


\allowdisplaybreaks
\begin{align*}
    p\left ( \mathbf{g} | \mathbf{S},\mathbf{t},\neg\tilde{\mathbf{t}}\right )
     &= \frac{p\left ( \mathbf{g},\mathbf{S},\mathbf{t},\neg\tilde{\mathbf{t}}\right )}{p\left ( \mathbf{S},\mathbf{t},\neg\tilde{\mathbf{t}}\right )} \\
     &\propto p\left ( \mathbf{g},\mathbf{S},\mathbf{t},\neg\tilde{\mathbf{t}}\right ) && \text{$p \left ( \mathbf{S}, \mathbf{t}, \tilde{\mathbf{t}}\right )$ is a constant} \\
     &= p \left ( \neg\tilde{\mathbf{t}}|\mathbf{g},\mathbf{t},\mathbf{S}\right ) p \left ( \mathbf{g},\mathbf{t},\mathbf{S}\right ) \\
     &= p \left ( \neg\tilde{\mathbf{t}}|\mathbf{g}\right )p \left ( \mathbf{g},\mathbf{t},\mathbf{S}\right ) && \text{$\tilde{\textbf{t}}$, $\textbf{t}$, $\textbf{S}$ are independent} \\
     &= p\left ( \neg\tilde{\mathbf{t}}|\mathbf{g}\right )p\left (\mathbf{t}, \mathbf{S}|\mathbf{g} \right )p\left ( \mathbf{g}\right ) \\
     &= p\left ( \neg\tilde{\mathbf{t}}|\mathbf{g}\right )p\left (\mathbf{t}, \mathbf{S}|\mathbf{g} \right )\frac{p\left ( \mathbf{g}|\mathbf{S}\right )p\left (\mathbf{S} \right )}{p\left ( \mathbf{S}|\mathbf{g}\right )} && \text{Using Bayes' Theorem} \\
     &\propto p\left ( \mathbf{g}|\mathbf{S}\right )p\left (\mathbf{t}, \mathbf{S}|\mathbf{g} \right )\frac{p\left ( \neg\tilde{\mathbf{t}}|\mathbf{g}\right )}{p\left ( \mathbf{S}|\mathbf{g}\right )} && \text{$p \left ( \mathbf{S}\right )$ is a constant} \\
     &= p\left ( \mathbf{g}|\mathbf{S}\right )p\left (\mathbf{t}, \mathbf{S}|\mathbf{g} \right ) \frac{p\left ( \mathbf{g},\neg\tilde{\mathbf{t}}\right )}{p\left ( \mathbf{g}\right )p\left ( \mathbf{S}|\mathbf{g}\right )} \\
     &= p\left ( \mathbf{g}|\mathbf{S}\right )p\left (\mathbf{t}, \mathbf{S}|\mathbf{g} \right ) \frac{p\left ( \mathbf{g}\right )-p\left ( \mathbf{g},\tilde{\mathbf{t}}\right )}{p\left ( \mathbf{g}\right )p\left ( \mathbf{S}|\mathbf{g}\right )} \\
     &= p\left ( \mathbf{g}|\mathbf{S}\right )p\left (\mathbf{t}, \mathbf{S}|\mathbf{g} \right ) \frac{1-p\left ( \tilde{\mathbf{t}}|\mathbf{g}\right )}{p\left ( \mathbf{S}|\mathbf{g}\right )} \\
     &\propto p\left ( \mathbf{g}|\mathbf{S}\right )p\left (\mathbf{t}, \mathbf{S}|\mathbf{g} \right ) \frac{1}{p\left ( \mathbf{S}|\mathbf{g}\right ) p\left ( \tilde{\mathbf{t}}|\mathbf{g}\right )} \\
     &= p\left ( \mathbf{g}|\mathbf{S}\right )p\left (\mathbf{t}, \mathbf{S}|\mathbf{g} \right ) \frac{p \left ( \mathbf{g}\right )}{p\left ( \mathbf{S}|\mathbf{g}\right ) p \left ( \mathbf{g},\tilde{\mathbf{t}}\right )} \\
     &= p\left ( \mathbf{g}|\mathbf{S}\right )p\left (\mathbf{t}, \mathbf{S}|\mathbf{g} \right ) \frac{p \left ( \mathbf{g}\right )}{p\left ( \mathbf{S}|\mathbf{g},\tilde{\mathbf{t}}\right ) p \left ( \mathbf{g},\tilde{\mathbf{t}}\right )} && \text{$\tilde{\textbf{t}}$, $\textbf{t}$, $\textbf{S}$ are independent} \\
     &= p\left ( \mathbf{g}|\mathbf{S}\right )p\left (\mathbf{t}, \mathbf{S}|\mathbf{g} \right ) \frac{p \left ( \mathbf{g}\right )}{p\left ( \mathbf{S},\mathbf{g},\tilde{\mathbf{t}}\right )} \\
     &= p\left ( \mathbf{g}|\mathbf{S}\right )\frac{p\left (\mathbf{t}, \mathbf{S}|\mathbf{g} \right )}{p \left (\tilde{\mathbf{t}},\mathbf{S}|\mathbf{g} \right )} \\
     &= p\left ( \mathbf{g}|\mathbf{S}\right )\frac{p\left (\mathbf{g}| \mathbf{t}, \mathbf{S} \right ) p \left ( \mathbf{t}, \mathbf{S} \right )}{p \left (\mathbf{g}|\tilde{\mathbf{t}},\mathbf{S} \right )p\left ( \tilde{\mathbf{t}}, \mathbf{S}\right )} && \text{Using Bayes' Theorem} \\
     &\propto p\left ( \mathbf{g}|\mathbf{S}\right )\frac{p\left (\mathbf{g}| \mathbf{t}, \mathbf{S} \right )}{p \left (\mathbf{g}|\tilde{\mathbf{t}},\mathbf{S} \right )} && \text{$p \left ( \mathbf{t}, \mathbf{S} \right )$, $p \left ( \tilde{\mathbf{t}}, \mathbf{S} \right )$ are constants}
 \end{align*}
 
The assumption of independence between $\tilde{\textbf{t}}$, $\textbf{t}$, and $\textbf{S}$ reflects general real-world scenarios where human language prompts can be arbitrary and are not necessarily dependent on the scene. Proposition 1 is now proved. $\blacksquare$
\end{proof}

\subsection{Connection between Diffusion and Energy-Based Models}
The connection between diffusion and energy-based models is not restricted to our problem. We will recall this connection in the general context of any generation task.

\textbf{Diffusion Models.}
Denoising diffusion probabilistic models (DDPMs) construct a forward diffusion process by gradually adding Gaussian noise to the ground truth sample $\mathbf{x}_0$ through T timesteps. A neural network then learns to revert this noise perturbation process. Both the forward and the reverse processes are modeled as Markov chains:
\begin{equation}
    q\left (\mathbf{x}_{0:T} \right )=q(\mathbf{x}_0)\prod_{t=1}^Tq\left (\mathbf{x}_t|\mathbf{x}_{t-1} \right), \quad p_\theta\left (\mathbf{x}_{T:0} \right )=p\left ( \mathbf{x}_T \right )\prod_{t=T}^1p_\theta\left (\mathbf{x}_{t-1}|\mathbf{x}_t \right ),
\end{equation}
where $q \left ( \mathbf{x}_0\right )$ is the ground truth data distribution and $p \left ( \mathbf{x}_T\right )$ is a standard Gaussian prior $\mathcal{N}\left (\mathbf{0}, \mathbf{I} \right )$.

In the reverse process, each step is parameterized by a Gaussian distribution with mean $\bm{\mu}_{\theta} \left ( \mathbf{x}_t, t\right )$ and covariance matrix $\tilde{\beta}_t\mathbf{I}$, where $\tilde{\beta}_t=\beta_t\frac{1-\bar{\alpha}_{t-1}}{1-\bar{\alpha}_t}$. Following the simplification in~\cite{ho2020denoising}, we can keep the covariance fixed and formulate the reverse distribution as:
\begin{equation}
    p_\theta\left (\mathbf{x}_{t-1}|\mathbf{x}_t \right )
    = \mathcal{N}\left ( \frac{1}{\sqrt{\alpha_t}}\left (\mathbf{x}_t-\frac{1-\alpha_t}{\sqrt{1-\bar{\alpha}_t}}\bm{\epsilon}_\theta\left(\mathbf{x}_t, t \right ) \right ),\beta_t\mathbf{I} \right ).
\end{equation}
Subsequently, an individual step in sampling can be performed by:
\begin{equation}
    \mathbf{x}_{t-1} = \frac{1}{\sqrt{\alpha_t}}\left ( \mathbf{x}_t - \frac{1-\alpha_t}{\sqrt{1-\bar{\alpha}_t}}\bm{\epsilon}_{\bm{\theta}}\left ( \mathbf{x}_t, t \right ) \right ) + \sqrt{\beta_t}\mathbf{z},
    \label{Eq: ddpm}
\end{equation}
where $\mathbf{z} \sim \mathcal{N}\left (\mathbf{0}, \mathbf{I} \right )$ if the time step $t>1$, else $\mathbf{z} = \mathbf{0}$.

\textbf{Energy-Based Models.}
Energy-Based Models (EBMs)~\cite{du2019implicit,du2021improved,nijkamp2020anatomy,grathwohl2020learning} are a family of generative models in which the data distribution is modeled by an unnormalized probability density. Given a sample $\mathbf{x}\in\mathbb{R}^D$, its probability density is defined as:
\begin{equation}
    p_\theta \left ( \mathbf{x}\right ) \propto e^{-E_\theta\left ( \mathbf{x}\right )},
\end{equation}
where the energy function $E_\theta\left ( \mathbf{x} \right ):\mathbb{R}^D\rightarrow\mathbb{R}$ is a learnable neural network. Langevin dynamics~\cite{du2019implicit} is then used to sample from the unnormalized probability distribution to iteratively refine the generated sample $\mathbf{x}$:
\begin{equation}
    \mathbf{x}_t = \mathbf{x}_{t-1}-\frac{\lambda}{2} \nabla_{\mathbf{x}}E_{\theta}\left (\mathbf{x}_{t-1} \right ) + \sqrt{\lambda}\mathbf{z},
    \label{Eq: ebm}
\end{equation}
where $\lambda$ is the predefined step size and $\mathbf{z} \sim \mathcal{N}\left (\mathbf{0}, \mathbf{I} \right )$.

The sampling procedure used by diffusion models in Equation~\ref{Eq: ddpm} is functionally similar to the sampling procedure used by EBMs in Equation~\ref{Eq: ebm}. In both settings, samples are iteratively refined starting from Gaussian noise, with a small amount of noise removed at each iterative step. At a timestep $t$, in DDPMs, samples are updated using a learned denoising network $\bm{\epsilon}\left ( \mathbf{x}_t, t\right )$, while in EBMs, samples are updated via the gradient of the energy function $\nabla_{\mathbf{x}}E_{\theta}\left (\mathbf{x}_{t} \right ) \propto \nabla_{\mathbf{x}}\log{p_\theta}\left (\mathbf{x}_t \right )$. Thus, we can view a DDPM as an implicitly parameterized EBM and apply similar composition techniques for EBMs as in~\cite{du2020compositional} for DDPMs. More details about compositional DDPMs can be referred to in~\cite{liu2022compositional}.
\section{Remark on Related Works}
\label{Sec: Rw}
\textbf{Diffusion Models in Robotics.}
Recent years have witnessed diffusion models being applied to several robotic tasks. For instance, in policy learning, diffusion models have been employ for multi-task robotic manipulation~\cite{xian2023unifying}, long-horizon skill planning~\cite{mishra2023generative}, or cross-embodiment skill discovery~\cite{xu2023xskill}. Besides, the ability of diffusion models to generate realistic videos over a long horizon has enabled new applications in the context of robotics~\cite{ajay2024compositional,du2024learning,ko2024learning}. For example, Du~\textit{et al.}~\cite{du2024learning} proposed to learn universal planning strategy via text-to-video generation. In robot development, diffusion models have been leveraged for manipulator construction~\cite{xu2024dynamics} or soft robot co-design~\cite{wang2023diffusebot}. Although diffusion models have also been explored for the task of grasp detection~\cite{urain2023se,nguyen2023language}, none of them address the task of detecting language-driven 6-DoF grasp poses in 3D cluttered scenes.

\textbf{Language-Driven Grasp Detection.}
Language-driven grasp detection has emerged as an active research domain in recent years. Previous works have primarily focused on addressing this task using 2D images~\cite{tang2023graspgpt,tziafas2023language,vuong2023grasp,xu2023joint,vuong2024language}. For instance, the authors in~\cite{tang2023ttask} presented a method that combines object grounding and task grounding to tackle the task of task-oriented grasp detection, while Xu~\textit{et al.}~\cite{xu2023joint} proposed to jointly model vision, language, and action for grasping in clutter. Despite achieving promising results, these approaches are limited in their ability to handle complex 3D environments. To overcome this limitation, recent research has explored language-driven grasp detection in 3D data. In particular, Nguyen~\textit{et al.}~\cite{nguyen2023language} addressed the task of affordance-guided grasp detection for 3D point cloud objects, while Tang~\textit{et al.}~\cite{tang2023graspgpt} leveraged knowledge from large language models for task-oriented grasping. However, these methods are designed for single-object scenarios, limiting their applicability in cluttered settings. In contrast, our method is capable of detecting language-driven 6-DoF grasp poses in cluttered point cloud scenes.
\section{Dataset Statistics}\label{Sec: Stats}
Table~\ref{tab:stats} shows our dataset statistics and comparisons to other 6-DoF grasp datasets.

\begin{table}[ht]
    \centering
    \resizebox{\linewidth}{!}{
    \begin{tabular}{l|c|c|c|c|c|c|c}
    \Xhline{1.0pt}
        Dataset & Text? & \#objects & \#grasps & \#scenes & Cluttered? & Data type & Annotation \\
        \hline
        GraspNet-1B~\cite{fang2020graspnet} & \ding{55} & 88 & $\sim$1.2B & 97K & \checkmark & Real & Analysis \\
        6-DoF GraspNet~\cite{mousavian20196} & \ding{55} & 206 & $\sim$7M & 206 & \ding{55} & Sim. & Sim. \\
        ACRONYM~\cite{eppner2021acronym} & \ding{55}  & 8872 & $\sim$17.7M & - & \ding{55} & Sim. & Sim. \\
        \hline
        \bf{Ours} & \checkmark & $\sim$3M & $\sim$200M & 1M & \checkmark & Synth. & Analysis \\
        \Xhline{1.0pt}
    \end{tabular}
    }
    \vspace{1ex}
    \caption{Dataset statistics.}
    \label{tab:stats}
\end{table}
\section{Implementation Details}
\label{Sec: Impl}
\subsection{Grasp Detection Methods for 3D Point Clouds}
\begin{itemize}
\item Our LGrasp6D: The text embedding $\mathbf{t}$ produced by the pretrained CLIP ViT-B/32 and the negative prompt embedding $\tilde{\mathbf{t}}$ are 512-dimensional (512-D). We employ a PointNet++~\cite{ni2020pointnet++} architecture for our scene encoder. The number of points per scene is 8192. The scene encoder extracts $n_S=128$ scene tokens of 256-D. We employ 4 heads for the multi-head cross-attention block, with the output of 512-D. The timestep $t$ is encoded by a sinusoidal positional encoder to obtain a 64-D vector. To speed up the training process, we freeze the scene encoder after the first 100 epochs.

\item 6-DoF GraspNet: We modified the model to integrate the text embedding derived from the CLIP text encoder~\cite{radford2021learning} into both the encoder and decoder of the variational autoencoder. Since our dataset does not include negative grasp poses, we refrained from employing additional refinement steps. This is also to ensure a fair comparison with other methods. The remaining architecture, hyperparameters, and training loss are inherited from the original work.

\item SE(3)-DF~\cite{urain2023se}: We append the text embedding extracted by the CLIP text encoder~\cite{radford2021learning} to the input of the feature encoder. As the signed distance function is not available for our 3D point clouds, we exclude the signed distance function learning objective from the framework. The remaining architecture, hyperparameters, and training loss are retained from the original work.

\item 3DAPNet~\cite{nguyen2023language}: 3DAPNet jointly addresses the tasks of language-guided affordance detection and pose detection. To adapt this method to our problem, we remove the affordance learning objective from the original framework. The remaining architecture, hyperparameters, and training loss are inherited from the original work.
\end{itemize}

\subsection{Grasp Detection Methods for Images}
Methods in this section are used in our robotic experiment in Section 5.2 of our main paper. They are trained on the RGB-D images to predict rectangle grasp poses inherited from Grasp-Anything~\cite{vuong2023grasp}. Specifically, each grasp pose is represented by $\left (g_x,g_y,g_w,g_h,g_\theta \right )$, where $(g_x, g_y)$ is the center of the rectangle, $(g_w, g_h)$ are the width and height of the rectangle and $g_\theta$ is the grasp angle.
\begin{itemize}
\item Language-supported versions of GG-CNN~\cite{morrison2018closing}, Det-Seg-Refine~\cite{ainetter2021end}, and GR-ConvNet~\cite{kumra2020antipodal}: We slightly modify these baselines by adding a component to fuse the input image and text prompt. Specifically, we utilize the CLIP text encoder~\cite{radford2021learning} to extract the text embedding. Additionally, we employ the ALBEF architecture presented in~\cite{li2021align} to fuse the text embedding and the visual features. The remaining training loss and architecture are inherited from the original works.

\item CLIPORT~\cite{shridhar2022cliport}: The original CLIPORT framework learns a policy $\pi$, which is not directly applicable to our setting. Therefore, we modify its architecture's final layers by adding an MLP to output the rectangle grasp pose.

\item CLIP-Fusion~\cite{xu2023joint}: We follow the cross-attention module in CLIP-Fusion. The final MLP in the architecture is modified to output five parameters of the rectangle grasp pose.

\item LGD~\cite{vuong2024language}: We report results from the original paper.

\end{itemize}
\section{Ablation Studies}
\label{Sec: abl}

\textbf{Negative Guidance Scale.} Recall that the negative guidance scale $w$ plays an important role in controlling the strength of the negative guidance in the sampling process. We conduct an ablation study of the effect of the change in $w$ on the grasp detection performance. Table~\ref{Tab:w} demonstrates that values of $w=0.2$ (used in experiments in the main paper) and $w=0.5$ yield the best results, whereas excessively small or large values of $w$ detrimentally affect performance.

\begin{table}[ht]
\renewcommand\arraystretch{1.2}
\renewcommand{\tabcolsep}{6.5pt}
\small
\centering
\begin{tabular}{c|ccc}\Xhline{1.0pt}
$w$ & \textbf{CR}$\uparrow$ & \textbf{EMD}$\downarrow$ & \textbf{CFR}$\uparrow$ \\ \hline
0.1 & 0.6573 & 0.4183 & 0.7629 \\ 
0.2 & \bf{0.6649} & \underline{0.4013} & \bf{0.7706} \\ 
0.5 & \underline{0.6607} & \bf{0.4005} & \underline{0.7698} \\ 
1.0 & 0.6531 & 0.4310 & 0.7622\\ 
2.0 & 0.6372 & 0.4521 & 0.7563 \\ 
\Xhline{1.0pt}
\end{tabular}
\vspace{1ex}
\caption{Grasp detection performance with varying negative guidance scale.}
\label{Tab:w}
\end{table}

\textbf{Loss Function.}
Table~\ref{tab:loss} shows the performances when using varying ratios of $\mathcal{L}_{\text{negative}}$ (called $\zeta$) and $\mathcal{L}_{\text{noise}}$ (which is $1-\zeta$). The results indicate that setting $\zeta$ to $0.1$ or $0.2$ yields strong accuracy, while either too high ($0.4$) or low ($0.05$) values significantly hurt the performance.

\begin{table}[ht]

\renewcommand\arraystretch{1.2}
\renewcommand{\tabcolsep}{6.5pt}
\small
\centering
\begin{tabular}{c|ccc}
    \hline
$\zeta$ & \textbf{CR}$\uparrow$ & \textbf{EMD}$\downarrow$ & \textbf{CFR}$\uparrow$ \\ \hline
0.05 & {0.6237} & {0.4500} & {0.7420} \\
0.1 & \bf{0.6733} & \bf{0.4029} & \underline{0.7754} \\ 
0.2 & \underline{0.6664} & \underline{0.4093} & \bf{0.7812} \\
0.4 & {0.5833} & {0.5298} & {0.7326} \\
\hline
    \end{tabular}
    \vspace{1ex}
    \caption{Loss function analysis.}
    \label{tab:loss}
\end{table}

\textbf{Backbone Variation.} We conduct an ablation study on two different scene encoder backbone, i.e., PointNet++~\cite{qi2017pointnet++} and Point Transformer~\cite{zhao2021point}, and two different pretrained text encoders, i.e., CLIP ViT-B/32~\cite{radford2021learning} and BERT~\cite{devlin2018bert}. The number of parameters and results of all variants are shown in Table~\ref{Tab:bb}. We observe that in general, PointNet++ performs better than Point Transformer, and CLIP performs better than BERT. Variants using Point Transformer run significantly slower than those using PointNet++ due to the larger and more complicated architecture. Particularly, the combination of Point Transformer and CLIP obtains a competitive grasp detection performance compared to that of PointNet++ and CLIP; however, its inference time is considerably higher. This pattern is also observed when comparing CLIP and BERT text encoders.  
The gap in grasp detection performance between variants utilizing the CLIP ViT-B/32 text encoder and those employing BERT is substantial, highlighting CLIP's superiority in semantic language-vision understanding.

\begin{table}[ht]
\renewcommand\arraystretch{1.2}
\renewcommand{\tabcolsep}{4.5pt}
\small
\centering
\resizebox{\textwidth}{!}{%
\begin{tabular}{r|c|cccc}\Xhline{1.0pt}
\textbf{Scene Encoder} & \textbf{Text Encoder} & \textbf{CR}$\uparrow$ & \textbf{EMD}$\downarrow$ & \textbf{CFR}$\uparrow$ & \textbf{IT}$\downarrow$ \\ \hline
Point Transformer~\cite{zhao2021point} (23M) & BERT~\cite{devlin2018bert} (110M) & 0.6428 & 0.4597 & 0.7583 & 2.0137 \\ 
Point Transformer~\cite{zhao2021point} (23M) & CLIP~\cite{radford2021learning} (63M) & \underline{0.6591} & \underline{0.4167} & \bf{0.7725} & 1.9755 \\ 
PointNet++~\cite{qi2017pointnet++} (2M) & BERT~\cite{devlin2018bert} (110M) & 0.6430 & 0.4225 & 0.7622 & \underline{1.5449} \\ 
PointNet++~\cite{qi2017pointnet++} (2M) & CLIP~\cite{radford2021learning} (63M)& \bf{0.6649} & \bf{0.4013} & \underline{0.7706} & \bf{1.4832} \\ 
\Xhline{1.0pt}
\end{tabular}
}
\vspace{1ex}
\caption{Scene encoder and text encoder backbone variation.}
\label{Tab:bb}
\end{table}
\vspace{-4ex}
\section{Robotic Experiments}
\label{Sec: Robot}
We show 20 real-world daily objects used in robotic experiments in Figure~\ref{Fig: Robot Obj}. The sequences of actions when the KUKA robot grasps objects are presented in Figure~\ref{Fig: Robot Seq}. Figure~\ref{Fig: Robot Det} further shows the detection result of our LGrasp6D on point clouds captured by a RealSense camera mounted on the robot. The robotic experiments demonstrate that although our method is trained on a synthetic Grasp-Anything-6D dataset, it can generalize to detect grasp poses in real-world scenarios. More illustrations can be found in our Demonstration Video.

\begin{figure}[h]
    \centering
    \includegraphics[width=\textwidth]{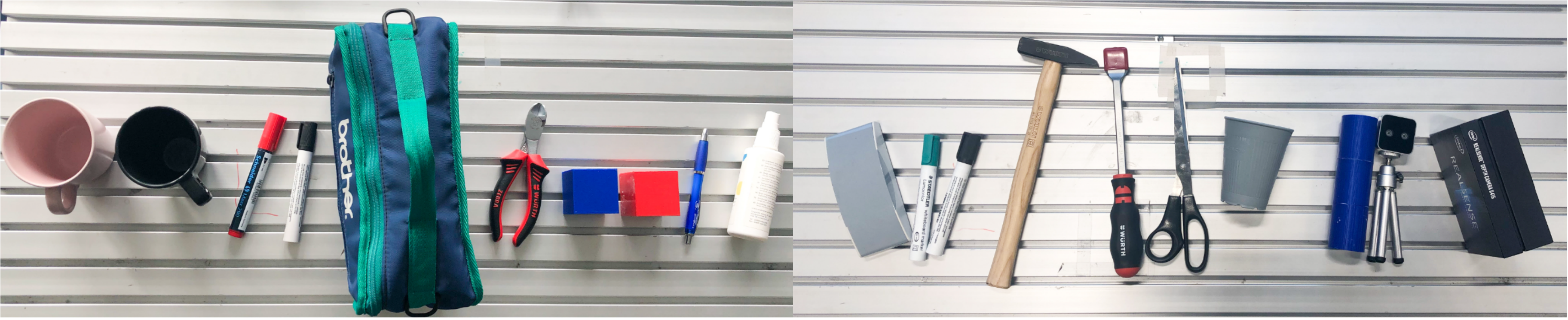}
    \caption{Set of 20 objects used in the robotic experiments.}
    \vspace{-3ex}
    \label{Fig: Robot Obj}
\end{figure}

\begin{figure}[ht]
    \centering
    \includegraphics[width=\textwidth]{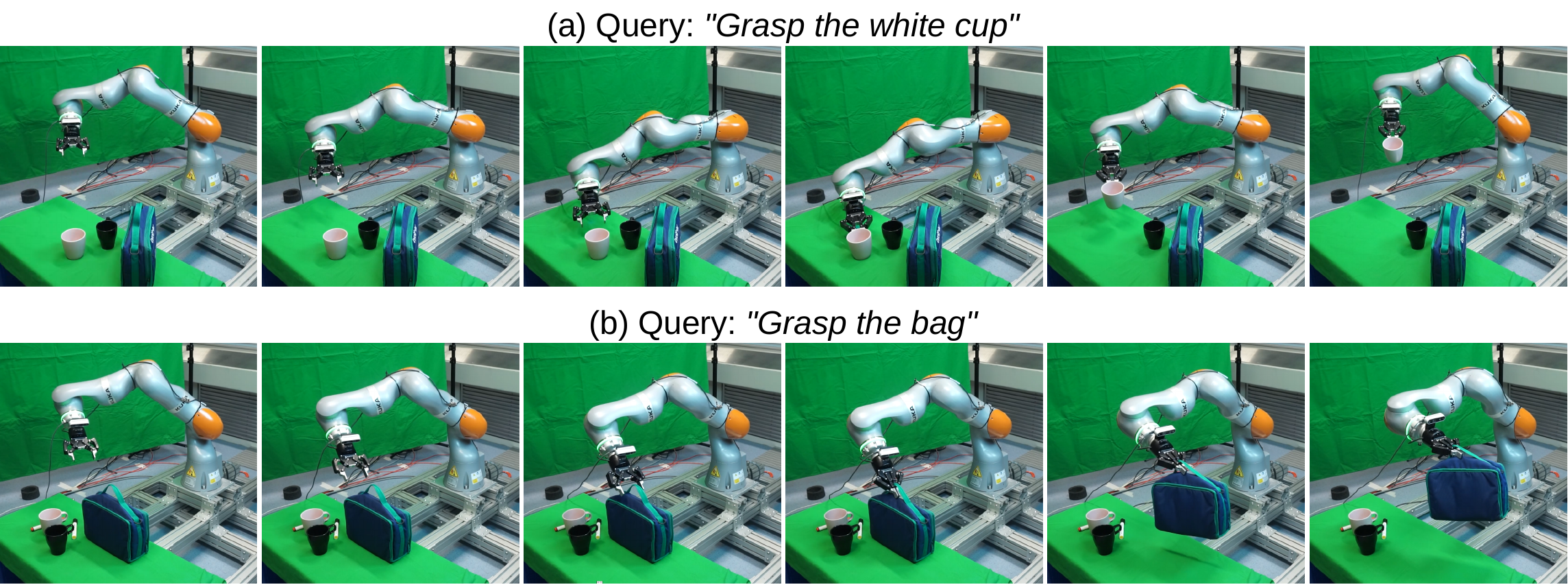}
    \vspace{-5ex}
    \caption{Snapshots of two example robotic experiments.}
    \label{Fig: Robot Seq}
\end{figure}
\vspace{-7ex}

\begin{figure}[ht]
    \centering
    \includegraphics[height=3.4cm]{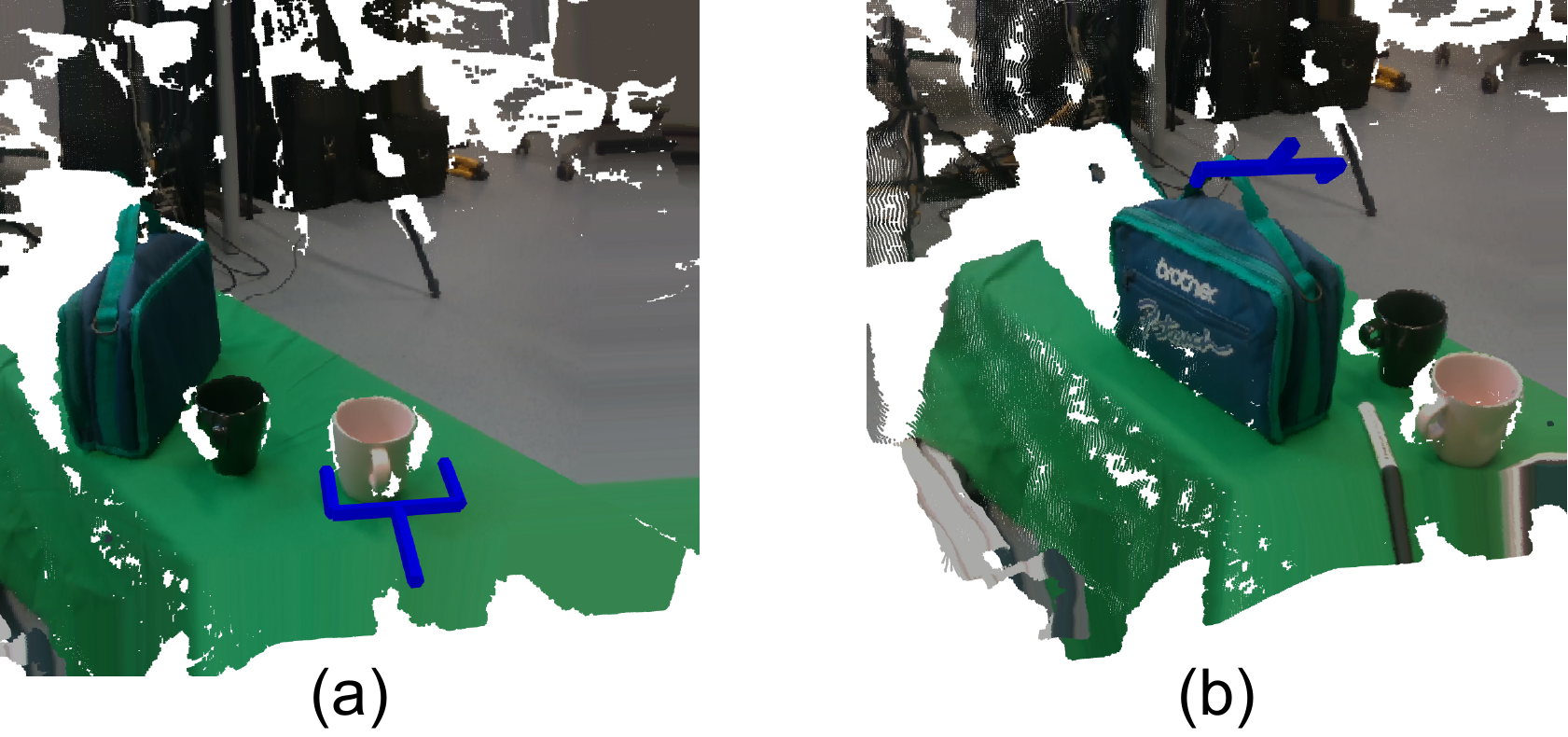}
    \vspace{-2ex}
    \caption{Detection results in robotic experiments. Point clouds are captured from a RealSense camera with experiments in Figure~\ref{Fig: Robot Seq}.}
    \label{Fig: Robot Det}
\end{figure}


\section{Additional Qualitative Results}
\label{Sec: Qual}
Figure~\ref{Fig:Add_Qual} shows more qualitative results to demonstrate the effectiveness of our method in detecting grasp poses for different objects.

\begin{figure}[hbt!]
    \vspace{-2ex}
    \centering
    \includegraphics[width=\textwidth]{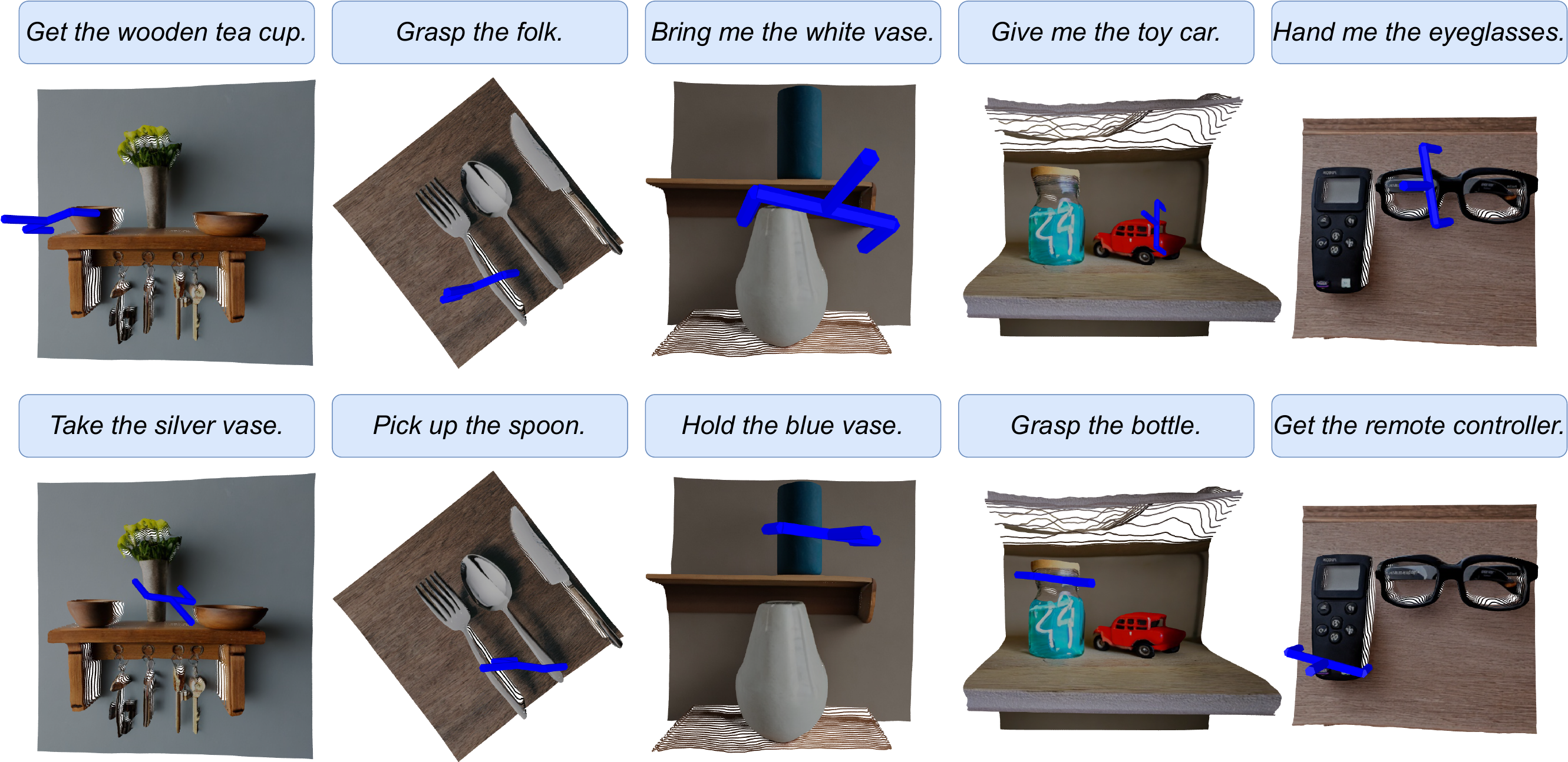}
    \vspace{-4ex}
    \caption{Additional qualitative results.}
    \label{Fig:Add_Qual}
\end{figure}

%
%
\bibliographystyle{ieee_fullname}
\bibliography{main}

\end{document}